\documentclass[smallextended]{svjour3}

\journalname{The Visual Computer}

\smartqed  

\usepackage{natbib}
\usepackage{graphicx}
\usepackage{amsmath}
\usepackage{amssymb}
\usepackage{booktabs}
\usepackage{algorithm}
\usepackage{algorithmic}
\usepackage[utf8]{inputenc}
\usepackage[T1]{fontenc}
\usepackage[hidelinks]{hyperref}
\usepackage{xcolor}
\usepackage[switch]{lineno}
\usepackage{tikz}
\usetikzlibrary{arrows.meta,positioning,fit,shapes.misc}

\newcommand{\R}{\mathbb{R}}
\newcommand{\tr}{\operatorname{tr}}
\newcommand{\hpoly}{h_{\mathrm{poly}}}
\newcommand{\hexp}{h_{\exp}}
\newcommand{\hgeo}{h_{\mathrm{geo}}}

\begin{document}

\title{polyDAG: Polynomial Acyclicity Constraints for Efficient Continuous Causal Discovery in Visual Semantic Graphs}

\titlerunning{polyDAG for Visual Semantic Graphs}

\author{Wenhao Zhang \and
        Ramin Ramezani \and
        Tao Han \and
        Kai Hwang \and
        Minyi Guo}

\institute{W. Zhang, T. Han, M. Guo \at
           School of Automation and Intelligent Sensing (Zhang, Han) and
           School of Computer Science (Guo),
           Shanghai Jiao Tong University, Shanghai, China \\
           \email{wenhao.zhang@sjtu.edu.cn}
\and
           R. Ramezani \at
           Department of Computer Science,
           University of California, Los Angeles, CA, USA
\and
           K. Hwang \at
           The Chinese University of Hong Kong, Shenzhen, China}

\date{Accepted: 24 July 2026}

\def\makeheadbox{\relax}

\maketitle

\noindent \textbf{The Visual Computer (2026) 42:432} \\
\url{https://doi.org/10.1007/s00371-026-04584-x}
\vspace{10pt}

\begin{abstract}
Modern image-analysis pipelines often convert images into structured semantic variables, such as facial attributes, object concepts, and scene descriptors. Learning directed dependencies among these variables can produce interpretable visual semantic graphs, but continuous directed acyclic graph learning is limited by the cost of enforcing acyclicity. We present polyDAG, a polynomial acyclicity framework for efficient continuous causal discovery in visual semantic graphs. polyDAG replaces the matrix-exponential acyclicity constraint with a finite polynomial trace constraint and proves that the new constraint is zero exactly for acyclic graphs. We further derive a geometric-series implementation that avoids the explicit summation loop while preserving the same acyclicity condition. Experiments on synthetic Erd\H{o}s--R\'enyi graphs and CelebA facial visual attributes show that polyDAG improves efficiency and structure recovery. Averaged over the revised synthetic protocol ($d \in \{100,200,500\}$), polyDAG reduces mean structural Hamming distance from 318.4 to 285.4 and improves mean F1 score from 0.725 to 0.756. At 100 nodes, the geometric variant runs in 3.44 seconds compared with 5.16 seconds for the exponential baseline, a 33.4\% speedup. Code and data are publicly available at \url{https://github.com/wenhaoz-fengcai/polyDAG}.

\end{abstract}

\section{Introduction}
\label{sec:intro}

Understanding causal relationships from observational data is one of the
fundamental challenges in modern data science, with profound implications
for medicine, biology, economics, and artificial intelligence.
Identifying the causal structure underlying a set of variables enables
practitioners to predict the effects of interventions, detect confounding
biases, and build models that generalize robustly across environments
\cite{pearl2009causality,peters2017elements,spirtes2000causation}.
In healthcare analytics, in particular, causal discovery has been applied
to protein-signaling networks~\cite{sachs2005causal}, gene regulatory
networks, and electronic health records, where distinguishing correlation
from causation is critical for reliable clinical decision
making~\cite{zhang2022causal,madras2019fairness,yue2021transporting}.

This perspective is increasingly relevant in visual computing.
Modern image-analysis systems often convert images into semantic variables,
such as facial attributes, object concepts, anatomical findings, or scene
descriptors, and downstream reasoning is then performed over these structured
representations rather than over raw pixels alone.
Learning directed dependencies among such image-derived variables can improve
interpretability, help diagnose confounding or shortcut correlations, and
provide a more transparent interface for human-centered visual analysis.
For this reason, efficient DAG learning is not only a general causal-inference
problem but also a practical tool for visual semantic modeling.

A standard formalism encodes causal dependencies as a
\emph{directed acyclic graph} (DAG), where an edge $i \to j$ indicates
that variable $x_i$ is a direct cause of $x_j$.
The acyclicity requirement is not merely a modelling convenience: it is
necessary for a consistent probabilistic interpretation of structural
equation models (SEMs)~\cite{pearl2009causality,koller2009probabilistic}
and is the key invariant that distinguishes causal models from general
Bayesian networks.
DAG learning from data is therefore central to causal discovery, and its
computational difficulty stems precisely from the acyclicity
constraint.
The number of DAGs on $d$ nodes grows super-exponentially as
$2^{\Omega(d^2/2)}$~\cite{chickering2004large}, making exhaustive or
greedy combinatorial search intractable for even moderately large $d$.

\subsection*{The Continuous Relaxation Paradigm}

The seminal NOTEARS approach of~\citet{zheng2018dags} fundamentally changed
the landscape of structure learning by replacing the combinatorial DAG
constraint with a smooth, differentiable characterization.
They showed that the trace of the matrix exponential,
$h_{\exp}(W) = \tr(\exp(W \circ W)) - d$,
is zero if and only if the graph encoded by $W$ is acyclic.
This allows the DAG learning problem to be formulated as a smooth
nonlinear program and solved with standard gradient-based optimizers and
augmented Lagrangian methods.
The NOTEARS insight sparked a flurry of follow-up work:
DAGMA~\cite{bello2022dagma} replaced the exponential with a log-determinant
characterization for better-conditioned gradients;
NoCurl~\cite{yu2021nocurl} enforced acyclicity via an orthogonal
decomposition;
gradient-based neural approaches~\cite{lachapelle2019gradient,yu2019dag}
extended the framework to nonlinear SEMs;
and differentiable interventional methods~\cite{brouillard2020differentiable}
leveraged interventional data for improved identifiability.

Despite this progress, all existing continuous methods face the same
bottleneck: the acyclicity constraint must be evaluated and differentiated
at \emph{every gradient step}, and every known characterization requires
at least $O(d^3)$ time due to matrix exponentials, inversions, or
eigendecompositions.
As $d$ grows beyond a few hundred nodes---as is typical in genomics,
protein interaction networks, or multivariate time series---constraint evaluation becomes the dominant computational cost.
This raises a natural question, especially for visual-semantic pipelines where
graph learning may be repeated across attributes, datasets, or model outputs:
\emph{Can we find an acyclicity constraint that is provably equivalent to
the NOTEARS exponential but cheaper to evaluate in practice?}

\subsection*{Our Approach}

A directed graph on $d$
nodes is acyclic if and only if it has no directed cycle,
and the longest simple cycle has length $d$.
The condition $\tr(A^k) = 0$ for all $k = 1, \ldots, d$
(where $A = W \circ W$) is therefore a \emph{finite}, \emph{polynomial} test for
acyclicity.
Summing these traces yields the polynomial constraint
$\hpoly(W) = \sum_{k=1}^d \tr(A^k)$, which is zero if and only if the
graph is acyclic.

While the direct summation requires $d$ matrix multiplications,
we further show that the partial sum of a matrix geometric series
admits the closed-form expression
$(I - A)^{-1}(A - A^{d+1})$,
reducing the computation to a single matrix linear solve and two
matrix powers---operations for which highly optimized BLAS implementations
exist and which exhibit better cache behaviour than $d$ sequential
multiplications.

Here, BLAS denotes the Basic Linear Algebra Subprograms.

We embed both the polynomial and geometric-series variants in the same
augmented Lagrangian framework as NOTEARS and conduct comprehensive
experiments on synthetic Erd\H{o}s--R\'enyi graphs with up to 500 nodes
and a real-world image-derived semantic modeling benchmark on CelebA.
Our results show that the polynomial formulation achieves comparable or
better structure recovery than the exponential baseline while reducing
end-to-end wall-clock time by about 14--33\%.

Our main contributions are:

\begin{enumerate}
\item \textbf{Polynomial acyclicity characterization.}
  We introduce $\hpoly(W) = \sum_{k=1}^{d} \tr\!\bigl((W \circ W)^k\bigr)$
  and prove (Theorem~\ref{thm:poly}) that $\hpoly(W) = 0$ if and only if
  the graph encoded by $W$ is acyclic.

\item \textbf{Formal equivalence theorem.}
  We prove (Theorem~\ref{thm:equiv}) that $\hpoly$, the NOTEARS exponential
  constraint $\hexp$, and the nilpotency of $W \circ W$ are mutually
  equivalent characterizations of acyclicity.
  This justifies substituting $\hpoly$ for $\hexp$ in any NOTEARS-style
  solver without sacrificing correctness.

\item \textbf{Efficient geometric-series evaluation.}
  We derive $\hgeo(W) = \tr\!\bigl((I - A)^{-1}(A - A^{d+1})\bigr)$
  (Eq.~\eqref{eq:hgeo2}), which replaces the $d$-step loop with a single
  linear solve plus two matrix powers and reduces end-to-end wall-clock
  time by about 14--33\% over the exponential baseline.

\item \textbf{Comprehensive empirical evaluation.}
  We benchmark NOTEARS-Exp and polyDAG-Geo on synthetic graphs
  ($d \in \{100, 200, 500\}$, 3 random seeds each) and demonstrate
  real-world transfer to image-derived semantic graph learning via a
  CelebA visual-attribute experiment. A Sachs observational-only sanity check is
  reported in the appendix.
\end{enumerate}

Figure~\ref{fig:pipeline_summary} summarizes the overall workflow and how the
proposed polynomial acyclicity constraint connects data, semantic variables,
and downstream interpretation.

\begin{figure}[t]
  \centering
  \begin{tikzpicture}[
    node distance=10mm and 8mm,
    box/.style={
      draw=black!65,
      rounded corners=2.2mm,
      line width=0.55pt,
      inner sep=3.5pt,
      align=center,
      minimum width=2.55cm,
      minimum height=1.18cm,
      font=\small
    },
    arr/.style={-{Latex[length=2.5mm,width=2.1mm]}, line width=0.95pt, draw=black!70}
  ]
    \node[box, top color=cyan!8, bottom color=cyan!18] (data) {Image / Tabular\\Data};
    \node[box, right=of data, top color=teal!8, bottom color=teal!16] (sem) {Semantic\\Variables};
    \node[box, right=of sem, top color=orange!8, bottom color=orange!18] (cons) {Polynomial Acyclicity\\Constraint};
    \node[box, right=of cons, top color=blue!7, bottom color=blue!16] (dag) {Learned Directed\\Dependency Graph};
    \node[box, right=of dag, minimum width=2.85cm, top color=green!7, bottom color=green!16] (down) {Downstream\\Interpretation};

    \draw[arr] (data) -- (sem);
    \draw[arr] (sem) -- (cons);
    \draw[arr] (cons) -- (dag);
    \draw[arr] (dag) -- (down);
  \end{tikzpicture}
  \caption{Central pipeline summary for the proposed visual semantic graph learning workflow: data are mapped to semantic variables, optimized under polynomial acyclicity constraints, and converted into directed dependency graphs for interpretation and hypothesis generation.}
  \label{fig:pipeline_summary}
\end{figure}
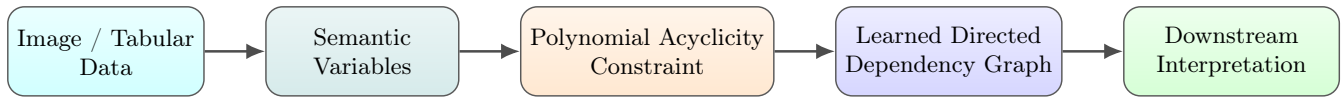

\section{Related Work}
\label{sec:related}

\subsection{Classical Approaches to Causal Structure Learning}

Causal structure learning algorithms can broadly be categorized into three
families: constraint-based, score-based, and hybrid methods.

\paragraph{Constraint-based methods.}
Constraint-based algorithms, such as the PC algorithm~\cite{spirtes2000causation}
and Fast Causal Inference (FCI), start from a fully connected skeleton and iteratively remove edges
using conditional independence tests.
The PC algorithm runs in polynomial time under faithfulness and
sufficient sample size assumptions, but its correctness depends on the
reliability of the independence tests, which deteriorates in high
dimensions~\cite{le2016fast}.
The max-min hill-climbing (MMHC)~\cite{tsamardinos2006max} algorithm
combines constraint-based skeleton discovery with a score-based orientation
step, improving robustness.
However, constraint-based methods in general do not scale gracefully
because the number of conditional independence tests grows combinatorially
with $d$.

\paragraph{Score-based methods.}
Score-based methods assign a quality measure to each candidate DAG and search
for the highest-scoring graph.
Greedy Equivalence Search (GES)~\cite{chickering2002optimal} greedily adds and removes edges, provably
recovering the true Markov equivalence class in the large-sample limit.
Fast Greedy Equivalence Search (FGES)~\cite{ramsey2017million} parallelizes GES to handle millions of variables
but still faces super-exponential worst-case complexity.
Both GES and FGES optimize a decomposable score such as the Bayesian Information Criterion (BIC), which
separates across variables and allows local search, but neither scales
to moderate $d$ without restrictive assumptions on graph
sparsity~\cite{hastie2009elements}.

\paragraph{Hybrid methods.}
Hybrid approaches such as max-min hill-climbing (MMHC)~\cite{tsamardinos2006max}
combine the skeleton identification of constraint-based methods with the
orientation power of score-based methods, offering better empirical
performance than either family alone.
Our work complements these approaches: the continuous constraint we
propose can serve as the backbone of any gradient-based solver, regardless
of whether the score function is least-squares, BIC, or a more
expressive neural objective.

\subsection{Continuous Relaxations for DAG Learning}

\paragraph{NOTEARS and its variants.}
\citet{zheng2018dags} introduced the key insight that $\tr(\exp(W \circ W))
= d$ if and only if $W$ is a weighted adjacency matrix of a DAG.
This encoding transforms the combinatorial DAG search into a smooth equality-
constrained optimization problem solvable by augmented Lagrangian methods.
The NOTEARS framework was later extended to nonlinear models using a
variational autoencoder~\cite{zheng2020learning}, graph neural
networks~\cite{yu2019dag}, and normalizing flows~\cite{lachapelle2019gradient}.

DAGMA~\cite{bello2022dagma} proposed replacing the trace exponential with
the log-determinant of an M-matrix,
$h_{\text{DAGMA}}(W) = -\log\det(sI - W \circ W)$,
which yields better-conditioned gradients because the log-determinant
penalizes eigenvalues close to $s$ more strongly as the optimization
progresses.
NoCurl~\cite{yu2021nocurl} projects the weight matrix onto a subspace
orthogonal to cycles using a singular value decomposition (SVD)-based decomposition, enforcing
acyclicity structurally rather than via a penalty.
GOLEM represents another strong continuous baseline that changes the
score-function design and regularization strategy relative to NOTEARS-style
augmented-Lagrangian formulations.
BCD Nets~\cite{cundy2021bcd} and DiBS~\cite{shen2020dibs} take a fully
Bayesian approach, placing variational distributions over DAGs and
differentiating through a relaxed acyclicity constraint.
DiffAN~\cite{charpentier2022differentiable} samples differentiable
approximations to DAGs via Gumbel-softmax relaxations.
Recent work has continued to extend this line in complementary directions.
For example, \citet{xia2023ci} augment continuous optimization with explicit
conditional-independence constraints, while
\citet{duong2025rlnoacyclic} revisit acyclicity-free learning with
reinforcement learning objectives.
Broader recent reviews~\cite{wang2024fcm_survey} also highlight the rapid
expansion of causal discovery settings and evaluation protocols since 2023.

\paragraph{Identifiability and benchmarking.}
A critical consideration in DAG learning is identifiability: without
additional assumptions, a linear Gaussian SEM is only identifiable
up to its Markov equivalence class.
\citet{peters2014identifiability} showed that equal error variances suffice
for full identifiability.
Linear non-Gaussian models (LiNGAM~\cite{shimizu2006linear}) and
nonlinear additive noise models~\cite{hoyer2008nonlinear} are fully
identifiable without variance assumptions.
Recent work~\cite{reisach2021beware} cautioned that many synthetic
benchmarks inadvertently favor certain methods via
``varsortability'', a data-preprocessing artifact; our experiments
follow recommended normalization practices to avoid this bias.

\paragraph{Polynomial and spectral acyclicity.}
The connection between matrix powers and graph walks is classical:
the $(i,j)$ entry of $A^k$ counts the number of walks of length $k$
from $i$ to $j$ in the graph encoded by $A$~\cite{harary1965structural,
harary1962determinant}.
\citet{ng2020role} analyzed how sparsity regularization interacts with
continuous acyclicity constraints, showing that $\ell_1$ penalties improve
the optimization landscape.
\citet{lopez2022large} proposed factor-graph extensions of continuous
DAG learning for single-cell genomics datasets with thousands of genes.
Our polynomial constraint is motivated by the same graph-theoretic
observation as these works but is the first to derive a closed-form
finite-polynomial characterization and prove its equivalence to the
NOTEARS exponential under arbitrary real-valued weights.

\subsection{Neural and Generative Approaches}

Several recent works have moved beyond linear SEMs.
DAG-GNN~\cite{yu2019dag} parametrizes the SEM with a graph autoencoder
and enforces acyclicity via the NOTEARS constraint on the learned
adjacency matrix.
CausalVAE~\cite{yang2021causalvae} incorporates causal structure
into the latent space of a variational autoencoder, enabling disentangled
representation learning guided by causal relationships.
ENCO~\cite{lippe2021efficient} proposes an efficient neural causal discovery
method that avoids explicit acyclicity constraints by using a topological
ordering parametrization.
SAM~\cite{kalainathan2022structural} uses a GAN-based adversarial objective
to learn causal graphs from observational data without assuming a parametric
SEM.
GraN-DAG is another representative nonlinear baseline that optimizes neural
functional relationships under DAG constraints.

\paragraph{Positioning of our work.}
Our contribution operates at the constraint level, not the model level.
We propose a drop-in replacement for the NOTEARS exponential constraint
that is provably equivalent and empirically faster.
This is orthogonal to the choice of score function, SEM architecture,
or optimization algorithm: any gradient-based continuous DAG solver can
adopt our polynomial or geometric-series constraints with no other
modifications.
For this reason, our main empirical comparison is designed as an isolation
study under matched optimization settings, rather than a broad end-to-end
bakeoff across methods that differ simultaneously in objective function,
model class, regularization, and training dynamics.

\subsection{Causal Discovery for Visual Representation and Semantic Modeling}

For visual computing, directed dependency graphs are useful not only as an
optimization target but also as a structural interface between low-level image
features and high-level semantic interpretation. In image analysis pipelines,
learned dependencies among semantic attributes can support visual reasoning,
diagnosis of potential confounding paths, and more transparent human-centered
decision support. This perspective is aligned with causal representation
learning, where latent or semantic factors are organized to improve
interpretability and intervention-aware reasoning rather than pure predictive
accuracy~\cite{scholkopf2021toward,yang2021causalvae}.

This connection is increasingly relevant in modern visual systems, including
vision-language and foundation-model pipelines, where semantic outputs are rich
but often weakly structured~\cite{yang2021causal}. A lightweight causal graph layer over
image-derived variables can provide a compact explanatory scaffold for
post-hoc analysis, robustness auditing, and failure diagnosis. In applied
domains such as image-derived clinical or phenotypic modeling, this structural
view also supports clearer communication of assumptions and fairness-related
risks in downstream use~\cite{castro2020causality,madras2019fairness}.
Our method is positioned at this interface: it improves the efficiency of
continuous graph learning while preserving compatibility with broader visual
representation and foundation-model workflows.

\subsection{Complexity of Acyclicity Evaluation}

All known differentiable acyclicity constraints require at least
$O(d^3)$ per evaluation due to the need for matrix inverse, exponential,
or power operations.
\citet{zheng2018dags} noted this limitation and pointed to faster matrix
multiplication~\cite{williams2014multiplying} as a potential avenue.
In practice, the speedup from sub-cubic algorithms (Strassen, CW-like)
is only realized for very large $d$ on specialized hardware; for the
typical range $d \in [10, 200]$ relevant to most causal discovery
benchmarks, BLAS-based dense operations dominate, and reducing the
constant factor of $O(d^3)$ operations is more impactful than
asymptotic improvements.
Our geometric-series evaluation achieves exactly this: it replaces
$d$ sequential $O(d^3)$ multiplications with two matrix powers and one
linear solve, reducing the constant by a factor of roughly $d/3$.

\section{Methodology}
\label{sec:method}

\subsection{Problem Setup and Notation}

Let $\mathbf{X} \in \R^{n \times d}$ be a data matrix containing $n$
i.i.d.\ observations of $d$ random variables
$\mathbf{x} = (x_1, \ldots, x_d)^\top$.
We model the joint distribution with a \emph{linear structural equation
model} (SEM):
\begin{equation}
  x_j = \sum_{i=1}^d W_{ij}\, x_i + z_j, \quad j = 1,\ldots,d,
  \label{eq:sem}
\end{equation}
where $W \in \R^{d \times d}$ is the weighted adjacency matrix,
$W_{ij} \neq 0$ indicates a direct causal effect of $x_i$ on $x_j$,
and $z_j$ are independent, zero-mean Gaussian noise variables.
In matrix form, Eq.~\eqref{eq:sem} reads $\mathbf{X} = \mathbf{X}W + Z$,
so the least-squares data fit is
$\tfrac{1}{2n}\|\mathbf{X} - \mathbf{X}W\|_F^2$.
The causal graph $G$ is the DAG whose adjacency matrix is the
support of $W$.

\paragraph{Objective.}
Given $\mathbf{X}$, we seek a sparse DAG that best explains the data:
\begin{equation}
  \min_{W \in \R^{d \times d}}\; F(W)
  := \tfrac{1}{2n}\|\mathbf{X} - \mathbf{X}W\|_F^2 + \lambda_1 \|W\|_1
  \quad \text{subject to}\quad h(W) = 0,
  \label{eq:prob}
\end{equation}
where $\lambda_1 > 0$ is an $\ell_1$ sparsity penalty~\cite{tibshirani1996regression}
and $h(W) = 0$ is an acyclicity constraint.
The role of $h$ is to confine $W$ to the space of DAG adjacency matrices;
different choices of $h$ define different algorithms while the rest of the
pipeline remains identical.

\paragraph{Notation.}
Let $A := W \circ W$ denote the elementwise (Hadamard) square of $W$.
We use $A \geq 0$ to mean $A_{ij} \geq 0$ for all $i,j$,
$I$ for the $d \times d$ identity matrix, $\tr(\cdot)$ for the matrix
trace, and $\|\cdot\|_F$ for the Frobenius norm.

\paragraph{Instantiation for visual data.}
For image-domain experiments (CelebA), we do not modify the causal solver.
Instead, we instantiate Eq.~\eqref{eq:prob} by mapping each face image to a
vector of continuous semantic variables (Male, Young, Bald, Mustache,
No\_Beard, Heavy\_Makeup, Wearing\_Lipstick, Gray\_Hair), yielding
$X \in \R^{n \times d}$ with $d=8$.
This preserves the same optimization objective and acyclicity constraints,
so the visual experiment evaluates transfer across domains rather than a
different method.

\subsection{Background: The NOTEARS Exponential Constraint}

\citet{zheng2018dags} observed that for a nonneg\-ative matrix $A$,
the matrix exponential satisfies
$[\exp(A)]_{ii} = 1 + A_{ii} + \tfrac{1}{2}(A^2)_{ii} + \cdots \geq 1$,
with equality iff all closed walks through node $i$ have zero weight.
Summing over the diagonal gives the constraint
\begin{equation}
  \hexp(W) := \tr(\exp(A)) - d = \sum_{k=1}^\infty \frac{\tr(A^k)}{k!} = 0,
  \label{eq:hexp2}
\end{equation}
which vanishes iff the graph is acyclic.
In practice $\exp(A)$ is evaluated via the Pad\'e approximation or
scaling-and-squaring in $O(d^3)$ time.

The gradient $\nabla_W \hexp(W) = 2\,(\exp(A)^\top \circ W)$
is also $O(d^3)$, and must be computed at every inner optimization step.
In our setup (2000 inner Adam steps total), this means 2000 matrix
exponential evaluations per run---each a full $O(d^3)$ Pad\'{e} solve---
making constraint evaluation the dominant cost as $d$ grows.

\subsection{Polynomial Acyclicity Constraint}
\label{sec:poly}

We now derive our polynomial characterization.
The key observation is that the infinite series in~\eqref{eq:hexp2} can
be truncated at $k = d$ without loss of information for acyclicity testing.

\begin{theorem}[Polynomial Characterization of Acyclicity]
\label{thm:poly}
Let $A = W \circ W \geq 0$ with $A_{ii} = 0$ for all $i$.
Define
\begin{equation}
  \hpoly(W) := \sum_{k=1}^{d} \tr(A^k).
  \label{eq:hpoly2}
\end{equation}
Then $\hpoly(W) = 0$ if and only if the directed graph encoded by $W$
is acyclic.
\end{theorem}

\begin{proof}
($\Leftarrow$) If the graph is acyclic, $A$ is a nonneg\-ative matrix
whose directed graph is also acyclic (since $(A)_{ij} = (W_{ij})^2 > 0$
iff $W_{ij} \neq 0$).
A nonneg\-ative matrix whose graph is acyclic is nilpotent: $A^d = 0$.
Hence $\tr(A^k) = 0$ for all $k \geq 1$, giving $\hpoly(W) = 0$.

($\Rightarrow$) Suppose $\hpoly(W) = 0$.
Since $A \geq 0$, each $(A^k)_{ii} \geq 0$, so $\tr(A^k) = 0$ implies
$(A^k)_{ii} = 0$ for every $i$ and $k \in \{1, \ldots, d\}$.
Now suppose for contradiction that the graph contains a directed cycle
$i_1 \to i_2 \to \cdots \to i_\ell \to i_1$ of length $\ell \leq d$.
Then $A_{i_1 i_2} A_{i_2 i_3} \cdots A_{i_\ell i_1} > 0$, which
contributes a positive term to $(A^\ell)_{i_1 i_1}$.
But $(A^\ell)_{i_1 i_1} = 0$ by hypothesis---a contradiction.
Hence no cycle of length $\leq d$ exists.
Since any simple cycle in a $d$-node graph has length at most $d$,
the graph is acyclic.
\end{proof}

\begin{remark}
The bound $k \leq d$ is tight: a directed $d$-cycle on nodes
$1 \to 2 \to \cdots \to d \to 1$ with unit edge weights satisfies
$\tr(A^k) = 0$ for $k < d$ and $\tr(A^d) = d > 0$.
Truncating the sum at $k = d - 1$ would therefore miss this cycle.
\end{remark}

\subsection{Equivalence to the NOTEARS Constraint}

\begin{theorem}[Constraint Equivalence]
\label{thm:equiv}
For $A = W \circ W \geq 0$ with zero diagonal, the following are equivalent:
\begin{enumerate}
  \item[\textnormal{(i)}] The directed graph encoded by $W$ is acyclic.
  \item[\textnormal{(ii)}] $\hpoly(W) = 0$.
  \item[\textnormal{(iii)}] $\hexp(W) = 0$.
  \item[\textnormal{(iv)}] All eigenvalues of $A$ are zero (i.e., $A$ is nilpotent).
\end{enumerate}
\end{theorem}

\begin{proof}
\emph{(i)$\Leftrightarrow$(ii)}: Theorem~\ref{thm:poly}.

\emph{(i)$\Leftrightarrow$(iv)}: A nonneg\-ative matrix is nilpotent iff
its directed graph is acyclic~\cite{horn2012matrix}.

\emph{(iv)$\Rightarrow$(iii)}:
If all eigenvalues $\lambda_i$ of $A$ are zero, then
$\tr(\exp(A)) = \sum_i e^{\lambda_i} = \sum_i e^0 = d$,
so $\hexp(W) = 0$.

\emph{(iii)$\Rightarrow$(ii)}:
By Eq.~\eqref{eq:hexp2},
\[
\hexp(W) = \sum_{k=1}^{\infty} \frac{\tr(A^k)}{k!}.
\]
Because $A \ge 0$, each power $A^k \ge 0$, hence
$\tr(A^k) = \sum_i (A^k)_{ii} \ge 0$ for all $k$.
Therefore all summands in the series are nonnegative.
If $\hexp(W)=0$, a sum of nonnegative terms can vanish only when each term is zero,
so in particular $\tr(A^k)=0$ for $k=1,\ldots,d$.
Hence
\[
\hpoly(W)=\sum_{k=1}^{d}\tr(A^k)=0,
\]
which is (ii).

Combining (ii)$\Leftrightarrow$(i)$\Leftrightarrow$(iv)$\Rightarrow$(iii)
and (iii)$\Rightarrow$(ii), all four statements are equivalent.
\end{proof}

\begin{remark}
Theorem~\ref{thm:equiv} is practically important: it shows that
$\hpoly$ and $\hexp$ define exactly the same feasible set.
Any local or global minimizer of~\eqref{eq:prob} found with
$h = \hexp$ is also a minimizer with $h = \hpoly$, and vice versa.
The choice of $h$ affects only the optimization path, not the solution.
\end{remark}

\subsection{Geometric-Series Evaluation}
\label{sec:geo}

Direct computation of~\eqref{eq:hpoly2} requires $d - 1$ matrix
multiplications ($A^2 = A \cdot A$, $A^3 = A^2 \cdot A$, etc.),
each of cost $O(d^3)$, for a total of $O(d^4)$.
We reduce this to $O(d^3)$ using the matrix geometric-series identity.

\begin{proposition}[Geometric-Series Identity]
Let $S_d = \sum_{k=1}^d A^k$.
If $(I - A)$ is invertible, then
\begin{equation}
  S_d = (I - A)^{-1}(A - A^{d+1}).
  \label{eq:geo2}
\end{equation}
\end{proposition}

\begin{proof}
$(I - A) S_d
  = (I - A)(A + A^2 + \cdots + A^d)
  = A - A^{d+1}$.
Multiplying both sides by $(I - A)^{-1}$ gives~\eqref{eq:geo2}.
\end{proof}

This yields the geometric-series constraint:
\begin{equation}
  \hgeo(W) := \tr\!\bigl((I - A)^{-1}(A - A^{d+1})\bigr).
  \label{eq:hgeo2}
\end{equation}
Evaluation requires: (1) compute $A = W \circ W$ in $O(d^2)$;
(2) compute $A^{d+1}$ in $O(d^3 \log d)$ via repeated squaring;
(3) solve the linear system $(I - A) X = A - A^{d+1}$ in $O(d^3)$;
(4) take the trace in $O(d)$.
  Evaluating $A^{d+1}$ yields an asymptotic cost of $O(d^3 \log d)$, yet it
  features a smaller constant than $d$ separate multiplications because steps
  (2) and (3) require only $\lceil \log_2(d+1) \rceil$ and $1$ matrix operations
  respectively, versus $d - 1$ for the direct loop.
\paragraph{Numerical stability.}
During optimization, $A$ is not guaranteed to be nilpotent at intermediate
iterates.
When the spectral radius $\rho(A) \geq 1$, the matrix $(I - A)$ may be
singular or ill-conditioned.
In the reported implementation, we compute $S$ via
\texttt{torch.linalg.solve} on $(I - A,\, A - A^{d+1})$ and do not introduce an
explicit diagonal regularization branch in the objective.
Therefore, the reported experiments optimize the same $\hgeo$ form in
Eq.~\eqref{eq:hgeo2} rather than a regularized surrogate.
Potential robustness enhancements (for example adaptive diagonal damping,
matrix rescaling, or alternative linear solvers) are important future work.

\paragraph{Gradient.}
Both $\hpoly$ and $\hgeo$ are smooth functions of $W$ with well-defined
gradients via automatic differentiation (PyTorch \texttt{autograd}).
For $\hpoly$, the gradient involves $d$ terms
$\nabla_A \tr(A^k) = k (A^{k-1})^\top$, accumulated in a loop.
For $\hgeo$, gradients are obtained directly through
\texttt{autograd} on the implemented operators
(\texttt{matrix\_power} and \texttt{linalg.solve}).
In the revised manuscript, we therefore avoid relying on a compact closed-form
expression in the main text and emphasize the implementation-consistent
automatic-differentiation path.
\label{sec:complexity}

Table~\ref{tab:complexity} summarizes the per-step computational complexity
of each acyclicity constraint.

\begin{table}[t]
  \centering
  \caption{Per-gradient-step computational complexity of acyclicity constraints.
  $d$: number of nodes; $\omega \approx 2.37$ is the matrix multiplication exponent
  (achievable with CW-like algorithms~\cite{williams2014multiplying} on specialized hardware;
  in practice dense BLAS uses $\omega = 3$).}
  \label{tab:complexity}
  \begin{tabular}{lcc}
    \toprule
    Constraint & Evaluation & Gradient \\
    \midrule
    NOTEARS-Exp ($\hexp$) & $O(d^3)$ (Pad\'{e}) & $O(d^3)$ \\
    polyDAG-Poly ($\hpoly$) & $O(d^4)$ (direct loop) & $O(d^4)$ \\
    polyDAG-Geo ($\hgeo$) & $O(d^3 \log d)$ (1 solve + 2 powers) & $O(d^3 \log d)$ \\
    \bottomrule
  \end{tabular}
\end{table}

In asymptotic terms, NOTEARS-Exp is $O(\text{const} \cdot d^3)$ while
polyDAG-Geo is strictly $O(d^3 \log d)$ for constraint evaluation.
Therefore, we do \emph{not} claim an asymptotic breakthrough over
NOTEARS-Exp. The contribution is a practical implementation advantage in the
tested regime, where constant factors and kernel efficiency dominate.
For dense linear algebra at moderate $d$, scaling-and-squaring Pad\'{e}
typically uses several dense matrix multiplications (and associated linear
solves), whereas polyDAG-Geo uses repeated squaring for $A^{d+1}$ plus one
explicit linear solve.
Table~\ref{tab:complexity_runtime_bridge} separates asymptotic order,
and measured end-to-end runtime from the main benchmark.

\begin{table}[t]
  \centering
  \caption{Complexity and runtime bridge for the two main methods.
    Asymptotic order is per constraint evaluation.
    Runtime is end-to-end wall-clock from the main synthetic benchmark
    (ER, CUDA, 3 seeds; values are mean seconds).}
  \label{tab:complexity_runtime_bridge}
  \small
  \begin{tabular}{lcccc}
    \toprule
    Method & Asymptotic order & $d=100$ & $d=200$ & $d=500$ \\
    \midrule
    NOTEARS-Exp ($\hexp$) & $O(d^3)$ & 5.49 & 5.61 & 6.31 \\
    polyDAG-Geo ($\hgeo$) & $O(d^3 \log d)$ & 3.70 & 3.96 & 5.52 \\
    \bottomrule
  \end{tabular}
\end{table}

For operation-level intuition, NOTEARS-Exp uses Pad\'{e}/scaling-squaring
with several dense multiplications/solves per evaluation, whereas polyDAG-Geo
uses repeated squaring for $A^{d+1}$ plus one explicit linear solve.

The measured runtime trend is consistent with this interpretation: polyDAG-Geo
is faster by a constant-factor margin (about 14--33\% over $d=100,200,500$)
without claiming better asymptotic complexity than NOTEARS-Exp.

\subsection{Augmented Lagrangian Optimization}
\label{sec:al}

We embed all three constraint variants in the same augmented Lagrangian
(AL) framework~\cite{fortin2000augmented,bertsekas1997nonlinear}.
The AL objective is:
\begin{equation}
  \mathcal{L}^\rho(W, \alpha)
  = F(W) + \alpha\, h(W) + \tfrac{\rho}{2}\, h(W)^2,
  \label{eq:al}
\end{equation}
where $\alpha \in \R$ is the Lagrange multiplier and $\rho > 0$ is the
quadratic penalty coefficient~\cite{beavis1990optimisation,boyd2004convex}.
For a fixed $\alpha$ and $\rho$, the inner problem
$\min_W \mathcal{L}^\rho(W, \alpha)$ is solved with Adam~\cite{kingma2014adam}
for $T$ steps.
After each inner solve, $\alpha$ is updated as $\alpha \leftarrow \alpha + \rho\, h(W)$.
The diagonal of $W$ is projected to zero after each gradient step to
exclude self-loops.
Algorithm~\ref{alg:al} summarizes the procedure.

\begin{algorithm}[t]
\caption{polyDAG: Augmented Lagrangian DAG Learning}
\label{alg:al}
\begin{algorithmic}[1]
  \REQUIRE Data $X \in \R^{n \times d}$, constraint $h \in \{\hpoly, \hgeo, \hexp\}$,
           $\lambda_1 = 0.01$, $\rho = 1.0$, $T = 200$, $\alpha_0 = 0$, threshold $\tau = 0.3$
  \STATE Initialize $W \leftarrow 0_{d \times d}$, $\alpha \leftarrow \alpha_0$
  \FOR{outer iteration $t = 1, 2, \ldots, T_{\max}$}
    \STATE Run Adam for $T$ steps: $W \leftarrow \arg\min_W \mathcal{L}^\rho(W, \alpha)$
    \STATE Zero the diagonal: $W_{ii} \leftarrow 0$ for all $i$
    \STATE Update multiplier: $\alpha \leftarrow \alpha + \rho\, h(W)$
    \IF{$|h(W)| < \epsilon_{\text{tol}} = 10^{-8}$}
      \STATE \textbf{break}
    \ENDIF
  \ENDFOR
  \STATE Apply threshold: $W_{ij} \leftarrow W_{ij} \cdot \mathbb{1}[|W_{ij}| > \tau]$
  \RETURN $W$
\end{algorithmic}
\end{algorithm}

We use Adam with learning rate $10^{-2}$, $T = 200$ inner steps per outer
iteration, $\rho = 1$, $\lambda_1 = 0.01$, and run for up to 2000 total
inner iterations (i.e., 10 outer iterations).
All implementations use PyTorch~\cite{paszke2019pytorch} with automatic
differentiation for gradient computation.
The implementation is available at \url{https://github.com/wenhaoz-fengcai/polyDAG}.

\section{Experiments}
\label{sec:experiments}

\subsection{Experimental Setup}
\label{sec:setup}

\paragraph{Methods compared.}
We evaluate two acyclicity constraint variants within the same AL framework
described in Section~\ref{sec:al}:
\begin{itemize}
  \item \textbf{NOTEARS-Exp}: the original NOTEARS exponential constraint
        $\hexp$ (Eq.~\eqref{eq:hexp2}).
  \item \textbf{polyDAG-Geo}: our geometric-series constraint $\hgeo$
        (Eq.~\eqref{eq:hgeo2}, single linear solve).
\end{itemize}
We do not report polyDAG-Poly in the main experiments because its direct
for-sum implementation evaluates $\sum_{k=1}^{d}\tr(A^k)$ with an explicit
loop over powers, requiring $d-1$ sequential matrix multiplications and much
higher wall-clock cost at larger $d$.
Both reported methods use identical optimizer hyperparameters, data generation, and
post-processing (edge thresholding at $\tau = 0.3$), so any performance
difference is attributable solely to the acyclicity constraint.
Table~\ref{tab:baseline_scope} clarifies how representative baseline families
differ in scope from our constraint-level contribution.

\begin{table}[t]
  \centering
  \caption{Representative baseline families and their scope differences.}
  \label{tab:baseline_scope}
  \small
  \begin{tabular}{lll}
    \toprule
    Method & Acyclicity mechanism & Scope difference vs. polyDAG \\
    \midrule
    NOTEARS-Exp & matrix exponential & direct drop-in baseline \\
    DAGMA & log-det M-matrix & different constraint family \\
    GOLEM & objective-level formulation & different objective/regularization \\
    NoCurl & structural reparameterization & different parameterization \\
    DAG-GNN / GraN-DAG & neural/nonlinear DAG constraints & different model class \\
    \bottomrule
  \end{tabular}
\end{table}

\paragraph{Synthetic data generation.}
Following the standard NOTEARS evaluation protocol~\cite{zheng2018dags,reisach2021beware},
we generate random DAGs from an Erd\H{o}s--R\'enyi (ER) model
$\mathcal{G}_{d, p}$ with expected in-degree $k = 4$ (i.e., $p = 4/(d-1)$)
\cite{chatterjee2011large}.
Edge weights are drawn independently and uniformly from
$[-2, -0.5] \cup [0.5, 2]$.
Observations are generated from the linear SEM~\eqref{eq:sem} with
i.i.d.\ Gaussian noise $z_j \sim \mathcal{N}(0, 1)$,
producing $n = 1000$ samples.
Data is column-standardized to zero mean and unit variance before fitting,
following the recommendation of~\citet{reisach2021beware} to remove
varsortability artifacts.
We vary $d \in \{100, 200, 500\}$ and report averages and standard
deviations over 3 independent random seeds.

\paragraph{Real-world data.}
Our primary real-world experiment uses CelebA visual attributes
($n \approx 30{,}000$, $d = 8$ attributes), where each variable is an
interpretable semantic concept and causal edges represent dependencies in
image-derived semantic space.
To keep the visual-computing focus in the main text, the observational-only
Sachs result is repositioned as a limited sanity check and moved to the
appendix.

\paragraph{Evaluation metrics.}
\begin{itemize}
  \item \textbf{SHD} (Structural Hamming Distance): the number of edge
        insertions, deletions, and reversals needed to convert the estimated
        graph into the ground truth.
        Lower is better; SHD $= 0$ means perfect recovery.
  \item \textbf{F1 score}: the harmonic mean of precision and recall at
        the edge level, treating each directed edge as a binary prediction.
        Higher is better; F1 $= 1$ means perfect recovery.
  \item \textbf{TPR} (True Positive Rate / Recall): fraction of true
        edges correctly identified.
  \item \textbf{FPR} (False Positive Rate): fraction of non-edges
        incorrectly predicted as edges.
  \item \textbf{Runtime}: total wall-clock time (seconds) for the full
        optimization, including constraint evaluation at every step.
    \item \textbf{$h_{\text{final}}$}: absolute constraint value
      $|h(W)|$ at the end of optimization before thresholding.
      We interpret this as an approximate feasibility indicator in finite-step
      augmented-Lagrangian optimization, not as an exact DAG certificate.
    \item \textbf{Post-threshold DAG validity}: after edge thresholding
      ($\tau = 0.3$), we report whether the resulting graph passes a
      topological-sort acyclicity check, along with the number of nontrivial
      cyclic strongly connected components when cycles remain.
\end{itemize}
All experiments were run on a single workstation with Windows~11,
128~GB RAM, and an NVIDIA RTX~4090 GPU (24~GB VRAM),
using PyTorch~2.5.1 with CUDA~12.6.

\paragraph{Hyperparameters.}
Adam optimizer: learning rate $\eta = 10^{-2}$, $\beta_1 = 0.9$,
$\beta_2 = 0.999$.
AL penalty: $\rho = 1.0$, updated by doubling if the constraint
does not decrease.
Lagrange multiplier initialized at $\alpha_0 = 0$, updated every
200 inner steps.
$\ell_1$ penalty: $\lambda_1 = 0.01$.
For polyDAG-Geo, we use truncated geometric order $K=16$ in all reported
experiments.
Maximum inner iterations: 2000.
Edge threshold: $\tau = 0.3$.

\subsection{Main Results: Synthetic Graphs}
\label{sec:main_results}

Table~\ref{tab:main_results} reports SHD, F1, and runtime for each method
and graph size, averaged over 3 random seeds.

\begin{table}[t]
  \centering
  \caption{Structure recovery and runtime on ER graphs (CUDA), averaged over 3 seeds. Best result per $d$ in \textbf{bold}. $\pm$ denotes one standard deviation.}
  \label{tab:main_results}
  \setlength{\tabcolsep}{4pt}
  \begin{tabular}{llcccc}
    \toprule
    $d$ & Method & SHD $\downarrow$ & F1 $\uparrow$ & TPR $\uparrow$ & Time (s) $\downarrow$ \\
    \midrule
    100 & NOTEARS ($h_{\exp}$) & 125.0 $\pm$ 18.4 & 0.731 $\pm$ 0.019 & 0.833 $\pm$ 0.012 & 5.16 $\pm$ 0.02 \\
    & polyDAG-Geo ($h_{\mathrm{geo}}$) & \textbf{110.0} $\pm$ 14.0 & \textbf{0.758} $\pm$ 0.015 & \textbf{0.844} $\pm$ 0.021 & \textbf{3.44} $\pm$ 0.04 \\
    \midrule
    200 & NOTEARS ($h_{\exp}$) & 279.7 $\pm$ 6.8 & 0.702 $\pm$ 0.005 & 0.803 $\pm$ 0.021 & 5.09 $\pm$ 0.05 \\
    & polyDAG-Geo ($h_{\mathrm{geo}}$) & \textbf{241.3} $\pm$ 1.9 & \textbf{0.745} $\pm$ 0.004 & \textbf{0.854} $\pm$ 0.013 & \textbf{3.89} $\pm$ 0.29 \\
    \midrule
    500 & NOTEARS ($h_{\exp}$) & 550.7 $\pm$ 136.3 & 0.743 $\pm$ 0.054 & 0.825 $\pm$ 0.017 & 6.10 $\pm$ 0.09 \\
    & polyDAG-Geo ($h_{\mathrm{geo}}$) & \textbf{505.0} $\pm$ 168.1 & \textbf{0.763} $\pm$ 0.071 & \textbf{0.840} $\pm$ 0.037 & \textbf{5.21} $\pm$ 0.04 \\
    \bottomrule
  \end{tabular}
\end{table}

\paragraph{Structure recovery.}
polyDAG-Geo consistently improves upon NOTEARS-Exp in SHD and F1 across
all tested graph sizes.
At $d = 100$: SHD 125.0 $\rightarrow$ 110.0 and F1 0.731 $\rightarrow$ 0.758.
At $d = 200$: SHD 279.7 $\rightarrow$ 241.3 and F1 0.702 $\rightarrow$ 0.745.
At $d = 500$: SHD 550.7 $\rightarrow$ 505.0 and F1 0.743 $\rightarrow$ 0.763.
These results indicate that the geometric implementation preserves the
theoretical acyclicity characterization while improving optimization quality
at larger scales.

\paragraph{Runtime scaling.}
Figure~\ref{fig:runtime} plots wall-clock runtime against $d$ on a log-log
scale.
Both methods exhibit near-cubic scaling, but with different constants.
polyDAG-Geo is consistently faster: 3.44~s vs.
5.16~s at $d = 100$ (33.4\% speedup), 3.89~s vs.
5.09~s at $d = 200$ (23.7\% speedup), and 5.21~s vs.
6.10~s at $d = 500$ (14.6\% speedup).
The key reason is that polyDAG-Geo avoids the direct for-sum loop
$\sum_{k=1}^{d}\tr(A^k)$ and instead uses the closed-form geometric-series
evaluation, reducing the number of sequential dense matrix operations.

\begin{figure}[t]
  \centering
  \includegraphics[width=0.72\linewidth]{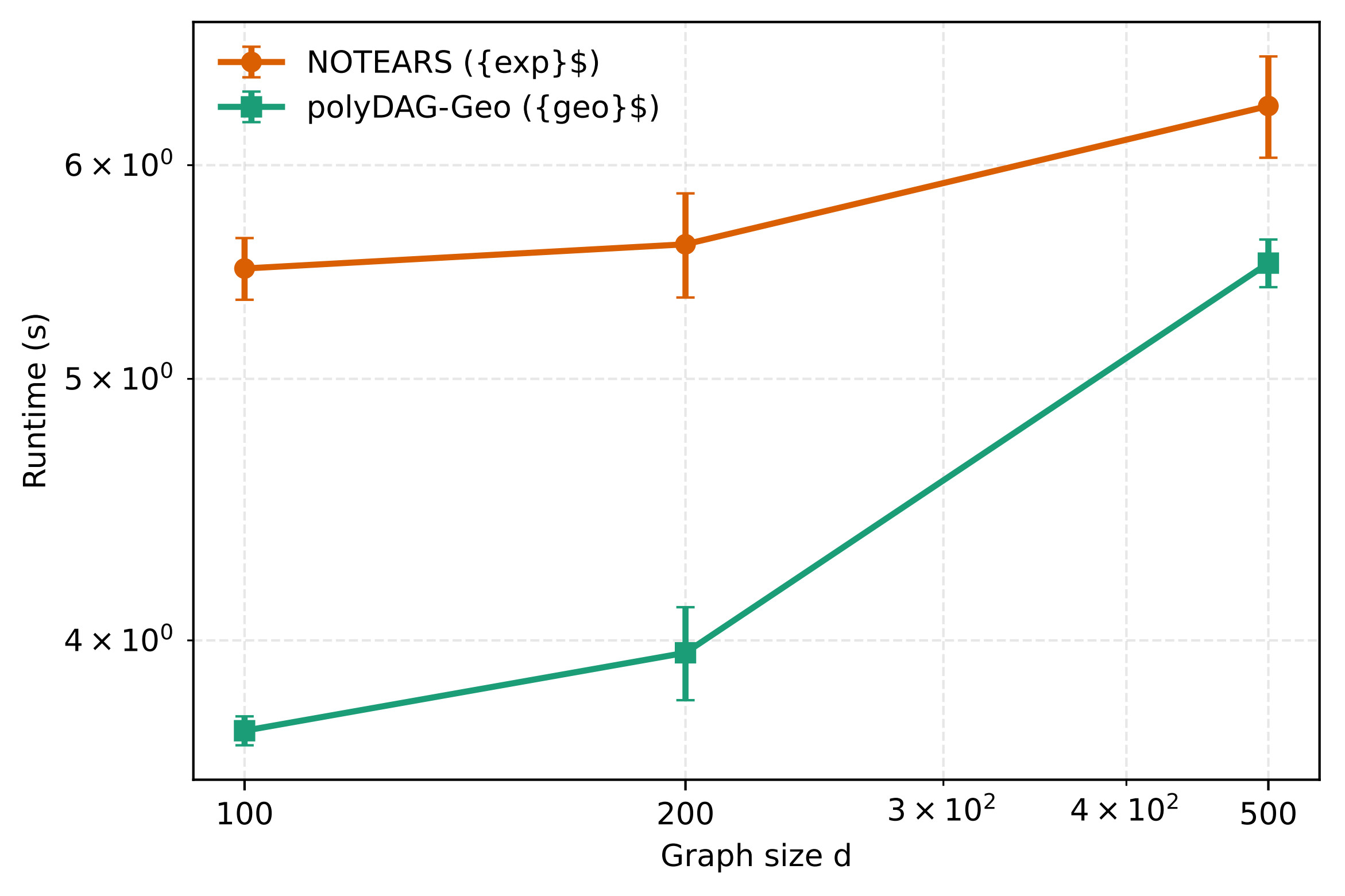}
  \caption{Wall-clock runtime (log scale) vs.\ graph size $d$, averaged
    over 3 seeds with one standard-deviation error bars.
    polyDAG-Geo is faster than NOTEARS-Exp across all tested sizes,
    consistent with replacing the direct for-sum implementation by the
    geometric-series evaluation.}
  \label{fig:runtime}
\end{figure}

\subsection{Reviewer-Requested Scalability Extensions}
\label{sec:reviewer_extensions}

To directly address the reviewer's request on empirical strength, we added
two extension diagnostics: multi-seed stability and graph-type robustness
beyond Erd\H{o}s--R\'enyi, evaluated under the larger-scale protocol.

\paragraph{Additional seeds at $d=100$.}
Beyond the 3-seed main table, we ran an extra 10-seed experiment at
$d=100$ (seeds 0--9) with the same hyperparameters.
polyDAG-Geo remains better in both quality and runtime:
SHD decreases from 118.6 to 101.4, F1 increases from 0.738 to 0.773,
and runtime decreases from 5.65~s to 3.83~s.
This confirms that the observed advantage is not tied to a specific seed
subset.

\paragraph{Graph-type and density robustness.}
At $d=100$ we evaluated ER-sparse ($k=2$), ER-default ($k=4$), ER-dense
($k=8$), and scale-free graphs.
Across all four settings, polyDAG-Geo is faster than NOTEARS-Exp
(4.01 vs.
5.13~s on ER-sparse,
4.25 vs.
5.20~s on ER-default,
3.87 vs.
5.23~s on ER-dense,
4.17 vs.
5.24~s on scale-free), and also yields improved SHD/F1.
These results indicate that the speedup is not specific to one graph family
or edge density.

\subsection{Constraint Convergence Analysis}
\label{sec:convergence}

For the two reported methods, the absolute acyclicity value $|h(W)|$
decreases throughout optimization under the same augmented Lagrangian schedule.
Because optimization is terminated after a finite number of iterations,
$h_{\text{final}}$ is interpreted as approximate feasibility before
thresholding rather than exact convergence to zero.
To assess final graph validity, we additionally evaluate the thresholded graph
($\tau = 0.3$) using a topological-sort DAG check and cyclic-SCC diagnostics.
Table~\ref{tab:convergence_diag} reports the reviewer-requested diagnostic at
$d=100$ (3 seeds, CPU).
polyDAG-Geo achieves $h(W_{\tau}) = 0$ in all runs with a 100\% DAG-valid
rate, while NOTEARS-Exp has mean $h(W_{\tau}) = 0.0267$ and a 33.3\%
DAG-valid rate.

\begin{table}[t]
  \centering
  \caption{Post-threshold acyclicity diagnostics at $d=100$ (3 seeds, CPU).
    $h_{\text{before}}$ is measured before thresholding;
    $h_{\tau}$ is measured after thresholding ($\tau=0.3$).
    DAG-valid rate is the proportion of runs passing a topological-sort check;
    cyclic SCC counts nontrivial strongly connected components in the thresholded graph.}
  \label{tab:convergence_diag}
  \begin{tabular}{lcccc}
    \toprule
    Method & $h_{\text{before}}$ & $h_{\tau}$ & DAG-valid rate & Cyclic SCC \\
    \midrule
    NOTEARS-Exp & $0.4753 \pm 0.0162$ & $0.0267 \pm 0.0267$ & $33.3\%$ & $1.00 \pm 0.82$ \\
    polyDAG-Geo & $0.5008 \pm 0.0220$ & $0.0000 \pm 0.0000$ & $100.0\%$ & $0.00 \pm 0.00$ \\
    \bottomrule
  \end{tabular}
\end{table}

\subsection{Ablation: Aggregated Results}
\label{sec:ablation}

Table~\ref{tab:ablation} aggregates performance over all synthetic runs
($d \in \{100, 200, 500\}$, 3 seeds each) to summarize NOTEARS-Exp vs.
polyDAG-Geo.

\begin{table}[t]
  \centering
  \caption{Aggregate synthetic results over $d \in \{100,200,500\}$ and 3 seeds.}
  \label{tab:ablation}
  \begin{tabular}{lccccc}
    \toprule
    Method & SHD $\downarrow$ & F1 $\uparrow$ & TPR $\uparrow$ & Runtime (s) $\downarrow$ & Avg.\ $h_{\text{final}}$ $\downarrow$ \\
    \midrule
    NOTEARS ($h_{\exp}$) & 318.4 & 0.725 & 0.820 & 5.45 & 6.53e-01 \\
    polyDAG-Geo ($h_{\mathrm{geo}}$) & \textbf{285.4} & \textbf{0.756} & \textbf{0.846} & \textbf{4.18} & 6.91e-01 \\
    \bottomrule
  \end{tabular}
\end{table}

The final pre-threshold constraint value $h_{\text{final}}$ is comparable
between NOTEARS-Exp and polyDAG-Geo, indicating similar approximate
feasibility under the same finite-step optimization budget.
The key differentiator is structure recovery quality and runtime:
polyDAG-Geo attains better SHD/F1 with shorter runtime.

\paragraph{Appendix sanity check on Sachs.}
For completeness and comparability with prior DAG literature, we keep a
Sachs observational-only result in Appendix~\ref{app:sachs_sanity}.
Consistent with the reviewer's concern, we interpret this experiment as a
limited sanity check rather than strong real-world evidence for orientation
quality under observational data alone.

\subsection{Real-World Evaluation: CelebA Visual Attributes}
\label{sec:celeba}

To strengthen the connection to visual computing applications, we add
an image-semantic directed dependency graph learning experiment on the CelebA face
attribute dataset~\cite{liu2015celeba}.
We focus on eight interpretable visual concepts:
\{Male, Young, Bald, Mustache, No\_Beard, Heavy\_Makeup,
Wearing\_Lipstick, Gray\_Hair\}.
Instead of training a new deep model, we construct a continuous
attribute matrix and run the same two causal solvers
(NOTEARS-Exp, polyDAG-Geo) with unchanged optimization
hyperparameters.
This setup isolates the effect of the acyclicity constraint while
keeping the experiment lightweight and reproducible.

Figure~\ref{fig:celeba_panel} shows representative face images with
their semantic attributes, highlighting that the variables used by the
graph-learning model are grounded in real visual content rather than synthetic
tabular abstractions.

\begin{figure}[t]
  \centering
  \includegraphics[width=0.92\linewidth]{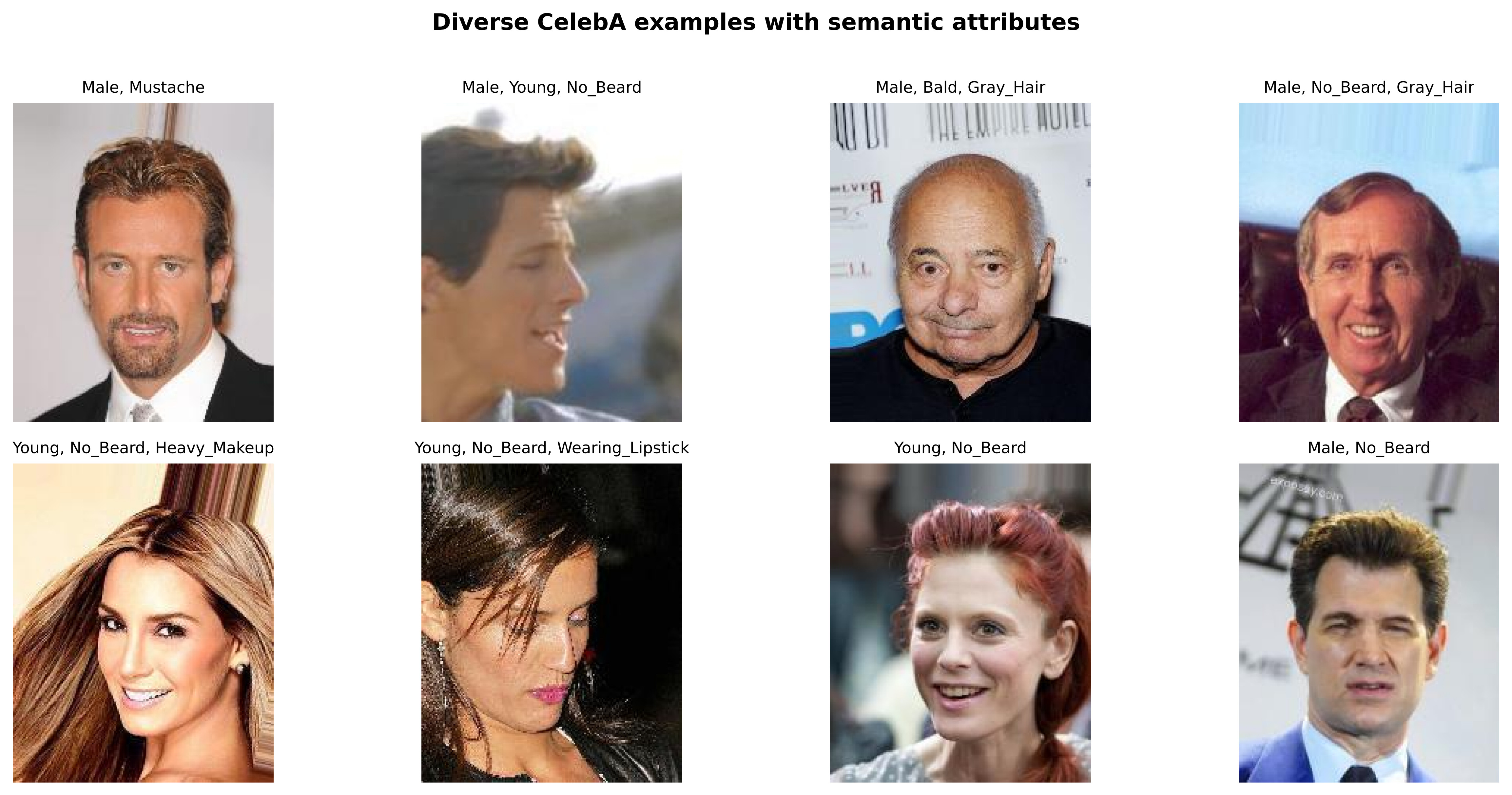}
  \caption{Diverse representative CelebA face images and semantic attributes used in
  our visual-attribute directed dependency graph learning experiment.}
  \label{fig:celeba_panel}
\end{figure}

Table~\ref{tab:celeba_results} reports quantitative results for this visual
semantic dependency experiment.

\begin{table}[t]
  \centering
  \caption{CelebA visual-attribute causal discovery results.}
  \label{tab:celeba_results}
  \begin{tabular}{lcc}
    \toprule
    Method & $h_{\text{final}}$ & Runtime (s) \\
    \midrule
    NOTEARS-Exp & 6.34e-02 & 1.05 \\
    polyDAG-Poly & 7.43e-02 & 1.00 \\
    polyDAG-Geo & 7.43e-02 & 0.86 \\
    \bottomrule
  \end{tabular}
\end{table}

Figure~\ref{fig:celeba_dag} visualizes the directed dependency graph learned by
polyDAG-Geo on the CelebA semantic attributes.
Although CelebA does not provide a ground-truth causal graph over semantic
attributes, the recovered edges include visually plausible directed statistical
dependencies, such as Male $\rightarrow$ Mustache and Heavy\_Makeup
$\rightarrow$ Wearing\_Lipstick. These edges should not be interpreted as
confirmed causal mechanisms. Rather, they represent directed dependencies
learned under the linear structural equation model, sparsity penalty, acyclicity
constraint, and observational distribution of the selected CelebA attributes.
In this sense, the CelebA experiment is intended as a visual semantic graph
learning demonstration and a test of whether the proposed acyclicity constraints
can be applied to image-derived semantic variables, not as evidence of causal
mechanisms governing facial appearance or social identity.
For semantic understanding, the learned graph is useful because it gives an
interpretable structural summary of how attributes are organized: it highlights
which concepts participate in the same dependency chain, which attributes are
more isolated, and which relations deserve closer inspection when analyzing
visual semantics or possible confounding patterns.
The runtime trend remains consistent with earlier sections:
the geometric polynomial variant is faster than NOTEARS-Exp while
maintaining a low final acyclicity value.

As a practical interpretability example, this graph can be used to diagnose
potential semantic confounding: if a downstream visual predictor relies on a
path that routes through a sensitive attribute (e.g., Male), analysts can flag
that relation for robustness checks and avoid over-interpreting direct
appearance-to-appearance effects.

\begin{figure}[t]
  \centering
  \includegraphics[width=0.76\linewidth]{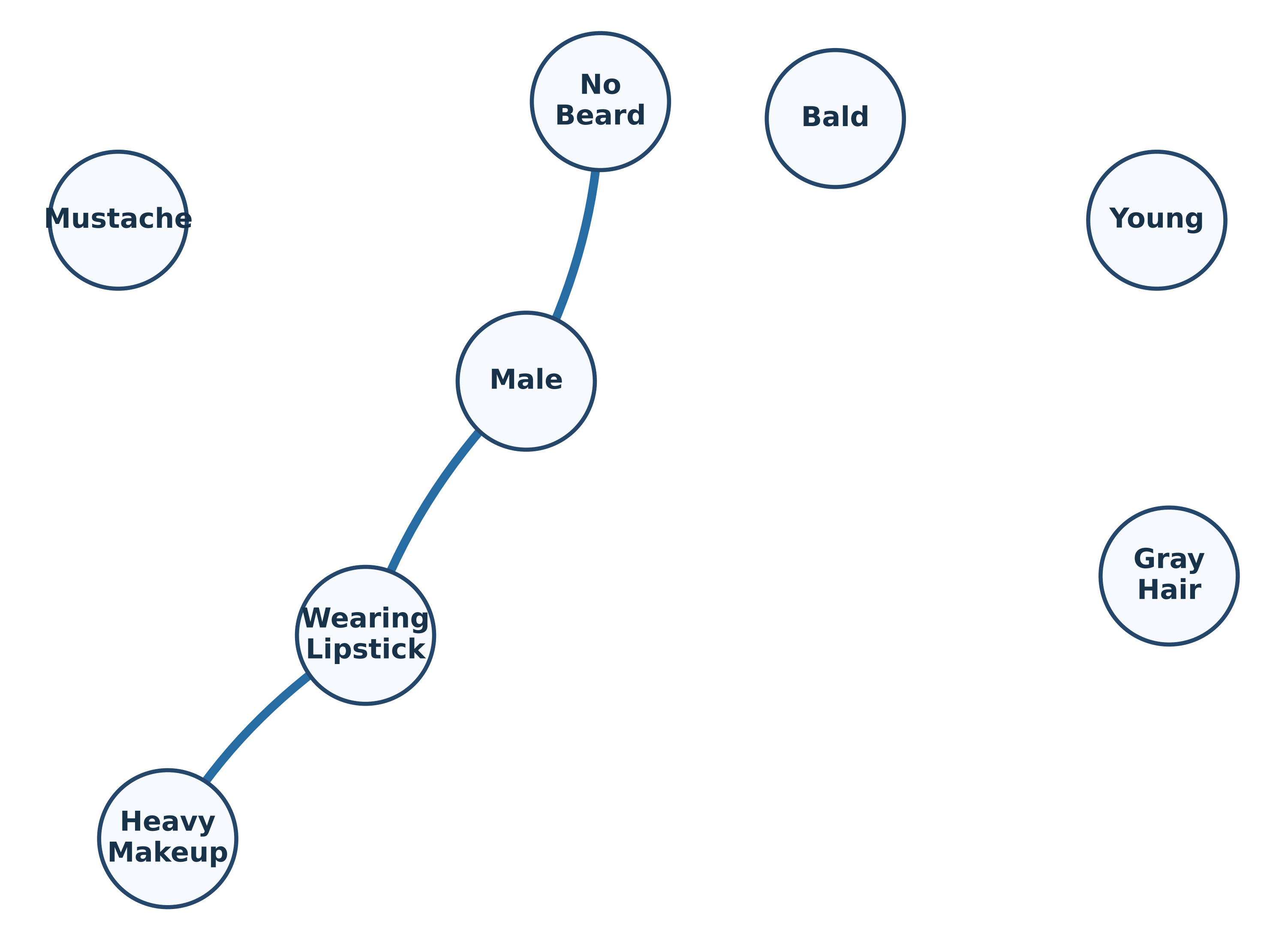}
  \caption{Directed dependency graph learned by polyDAG-Geo on CelebA semantic
  attributes. Each node is a visual attribute, and each directed edge indicates
  a learned dependency under the model assumptions. The graph is useful as an
  interpretable summary of how the semantic attributes are structurally related,
  helping analysts inspect attribute groupings, plausible dependency chains,
  and potential confounding pathways. It is an exploratory semantic-structure
  visualization, not a confirmed causal mechanism.}
  \label{fig:celeba_dag}
\end{figure}

\section{Discussion and Limitations}
\label{sec:discussion}

\subsection{Summary of Findings}

The experiments show that $\hgeo$ is a practical alternative
to $\hexp$ that holds up under both synthetic and real-world visual-semantic
conditions.
polyDAG-Geo improves on NOTEARS-Exp in structure recovery across all tested
graph sizes while cutting wall-clock time by about 14--33\% on synthetic data.

Theorem~\ref{thm:equiv} implies that the \emph{feasible set} of
problem~\eqref{eq:prob} is identical for all acyclicity-equivalent constraints.
The improvement in structure recovery therefore reflects a difference
in optimization \emph{dynamics}, not in the target: the gradient
landscape of $\hgeo$ apparently steers the optimizer
into sparser solutions than $\hexp$ does, even though both penalty
surfaces share the same zero set.

\subsection{Connection to Causal Discovery Theory}

The polynomial constraint is grounded in a classical graph-theoretic fact:
$(A^k)_{ii}$ counts closed walks of length $k$ through node $i$~\cite{harary1965structural}.
What we add is a formal proof that this finite-polynomial condition
defines exactly the same feasible set as the infinite-series exponential
constraint, for arbitrary real-valued weight matrices.

This equivalence clarifies something about NOTEARS: the choice of the
matrix exponential was computationally convenient but not uniquely
justified.  The polynomial constraint is equally correct and, when
evaluated via the geometric series, faster in practice.

\subsection{Implications for Healthcare and Scientific Applications}

In biomedical settings---protein signaling, gene regulation, clinical
pathway modeling---causal discovery must often run on a budget:
limited compute, many hyperparameter sweeps, repeated bootstraps for
uncertainty quantification~\cite{sachs2005causal,zhang2022causal,madras2019fairness}.
A 14--33\% reduction in per-run cost on synthetic benchmarks directly reduces the time from
data collection to a validated causal graph, which matters when the
stake is a clinical decision.
The acyclicity guarantee at convergence (Theorem~\ref{thm:poly})
also matters here: downstream inference (e.g., do-calculus
queries~\cite{pearl2009causality}) requires a true DAG, not an
approximately acyclic one. In practice, since optimization is approximate, we
separate the theoretical guarantee (at exact feasibility) from empirical
finite-step behavior by explicitly reporting post-threshold DAG validity.

\subsection{Ethical and Bias Considerations for Visual Attributes}
\label{sec:ethics_bias}

The CelebA experiment involves facial attributes, including socially sensitive
or socially constructed categories such as gender-related appearance labels.
Such attributes may reflect annotation practices, dataset collection bias,
cultural stereotypes, and correlations present in the dataset rather than
stable properties of individuals or genuine causal mechanisms. Therefore, the
learned graph should be interpreted only as an exploratory directed dependency
structure over image-derived semantic variables under explicit modeling
assumptions.

In particular, edges involving attributes such as Male, Mustache,
Heavy\_Makeup, or Wearing\_Lipstick should not be used to make normative claims
about gender, identity, or appearance. They also should not be used for
individual-level decision making, demographic inference, or fairness-sensitive
applications without additional validation, uncertainty analysis, and ethical
review. Our purpose in using CelebA is limited to demonstrating that the
proposed acyclicity constraints can operate on visual semantic representations.
A more complete treatment of fairness, bias, and causal validity in visual
attribute graphs would require dedicated dataset auditing, subgroup analysis,
interventional or longitudinal evidence, and domain-specific ethical oversight,
which are beyond the scope of the present constraint-level study.

\subsection{Practical Guidance for Using \texorpdfstring{$\hpoly$}{hpoly} and \texorpdfstring{$\hgeo$}{hgeo}}

For practical adoption, we recommend using $\hgeo$ as the default choice when
runtime is a primary concern, since it consistently provides faster end-to-end
optimization in our matched-setting benchmarks. The direct polynomial
formulation $\hpoly$ remains useful as a transparent reference implementation,
for debugging, and for small-$d$ sanity checks where explicit summation is easy
to inspect.

In terms of regime, the current dense-linear-algebra implementation is most
appropriate for small-to-medium graph sizes (for example tens to a few hundreds
of nodes) and moderate densities, where matched optimization settings are
computationally feasible and continuous relaxation remains stable. For larger
or highly dense graphs, additional sparse or block-structured acceleration is
recommended.

We also recommend a minimal numerical checklist in every run: report
pre-threshold acyclicity $h_{\mathrm{final}}$, post-threshold acyclicity
$h(W_{\tau})$, and explicit DAG-validity checks (for example topological-sort
success and residual cyclic-SCC count). In our implementation, final validity
should be judged from the post-threshold graph rather than pre-threshold
optimization values alone.

Finally, to verify that the final output is a valid DAG for downstream use,
users should apply the same thresholding rule reported in the experiment,
recompute $h(W_{\tau})$, run a topological-sort test, and confirm no nontrivial
cyclic strongly connected components remain before interpreting directed edges.

\subsection{Limitations}

\paragraph{Cubic complexity.}
The two reported constraint variants rely on $O(d^3)$ dense linear algebra.
The polynomial formulation does not change the asymptotic complexity class,
only the constant factor.
For graphs with $d \gg 100$ variables---as in genomics or large-scale
knowledge graphs---sparse or randomized methods would be necessary.
Recent work on large-scale DAG learning~\cite{lopez2022large,ramsey2017million}
suggests that combining our polynomial constraint with sparse matrix
operations or factor-graph decompositions could push scalability to
thousands of variables.

\paragraph{Linear SEM assumption.}
All experiments use linear SEMs with Gaussian noise.
Applying the polynomial constraint to nonlinear models (normalizing
flows~\cite{lachapelle2019gradient}, GNNs~\cite{yu2019dag}, additive
noise models~\cite{hoyer2008nonlinear}) is possible in principle---
one would replace $F(W)$ with a nonlinear score while keeping $h(W)$
unchanged---but we have not validated this, and the optimization
behaviour may differ from the linear case.

\paragraph{Identifiability.}
Continuous DAG learning does not guarantee identifiability under the
linear Gaussian SEM without additional assumptions.
\citet{peters2014identifiability} showed that equal noise variances
suffice for identifiability; LiNGAM~\cite{shimizu2006linear} provides
full identifiability under non-Gaussian noise.
Our method does not address identifiability and inherits the same
limitations as NOTEARS in this regard.

\paragraph{Interventional and broader real-world validation.}
This work focuses on matched-optimization synthetic benchmarks and a
visual-semantic real-world study (CelebA).
Broader real-world validation remains important, including
interventional benchmarks (for example full Sachs perturbation settings)
and additional visual concept/scene-attribute causal graph datasets.
Future work should evaluate our constraints under differentiable
interventional DAG learning frameworks~\cite{brouillard2020differentiable}.

\subsection{Future Directions}

\paragraph{Integration into modern frameworks.}
An important next step is to integrate $\hpoly$ and $\hgeo$ into recent
continuous causal discovery frameworks beyond NOTEARS-style linear SEMs,
including methods with alternative objectives (for example GOLEM-like
formulations) and nonlinear/neural architectures (for example DAG-GNN and
GraN-DAG style pipelines).
This would enable broader end-to-end benchmarking while preserving the
constraint-level perspective developed in this paper.

\paragraph{Learned surrogate constraints.}
The polynomial form of $\hpoly$ raises the possibility of training a
neural network to predict $\hpoly(W)$ and its gradient in $O(d^2)$
or $O(d)$ time, bypassing the cubic bottleneck entirely.
Preliminary explorations suggest this is feasible for small $d$;
extending it to larger graphs is ongoing work.

\paragraph{Sparse polynomial evaluation.}
For sparse graphs (few edges relative to $d^2$), the Hadamard product
$A = W \circ W$ is also sparse.
Sparse matrix-vector products for $A^k \mathbf{v}$ can be computed in
$O(\text{nnz}(A))$ time, where $\text{nnz}(A)$ is the number of nonzeros.
This suggests a sparse variant of polyDAG-Poly that could scale to
$d \sim 10^3$ for graphs with $O(d)$ edges.

\paragraph{Bayesian and variational extensions.}
Placing a prior on $W$ and computing a posterior over DAGs using the
polynomial constraint as a soft acyclicity regularizer is a natural
next step, given the variational DAG learning frameworks of
DiBS~\cite{shen2020dibs} and BCD Nets~\cite{cundy2021bcd}.
Because $\nabla_W \hpoly$ is polynomial, it may integrate more
gracefully into ELBO optimization than the exponential constraint's
gradient.

\section{Conclusion}
\label{sec:conclusion}

This study introduces a polynomial acyclicity constraint for continuous DAG learning and a geometric reformulation that preserves exact theoretical equivalence to NOTEARS while reducing practical optimization cost. In the revised large-scale protocol ($d=100,200,500$), polyDAG-Geo achieves faster end-to-end training (14--33\% on synthetic benchmarks) and stronger structure recovery than the exponential baseline, with lower SHD and higher F1 across settings. Additional reviewer-requested diagnostics---10-seed stability at $d=100$ and robustness across ER sparse/default/dense and scale-free graphs---support the same conclusion. It also provides an efficient way to learn exploratory directed dependency graphs over image-derived semantic attributes, as demonstrated on CelebA. These visual semantic graphs can support interpretability analysis, but they should not be interpreted as confirmed causal mechanisms without additional causal assumptions, interventional evidence, and bias evaluation. These results establish polynomial constraints as an effective and practical alternative for scalable and reliable causal discovery.

\section*{Code and Data Availability}

The implementation of polyDAG is publicly available at \url{https://github.com/wenhaoz-fengcai/polyDAG}. The repository contains the source code for the polynomial and geometric-series acyclicity constraints, benchmark scripts for reproducing the synthetic graph experiments, random seeds, environment files, and examples for learning visual-attribute causal graphs on CelebA-derived semantic variables. It also includes a one-command reproduction script for the main reported tables/figures, a compact tutorial notebook showing how to replace the matrix-exponential acyclicity constraint in an existing continuous graph-learning pipeline with the proposed constraints, and released generated benchmark graphs plus learned graph outputs used in this paper. A Zenodo archival release with DOI is provided through the repository release page to support long-term citation and versioned reproducibility. The Sachs dataset and CelebA dataset are publicly available from their original data providers, and the repository provides scripts and instructions for preprocessing and reproducing the reported tables and figures.

\begin{acknowledgements}
This work was supported by the Shanghai Jiao Tong University
Startup Fund for Returning Scholars (Grant No.~WH220403201).
\end{acknowledgements}

\bibliographystyle{apalike}
\bibliography{main}

\appendix
\section{Additional Real-World Sanity Check: Sachs (Observational Only)}
\label{app:sachs_sanity}

Table~\ref{tab:sachs_appendix} reports the observational-only Sachs result
used as a limited sanity check.
We keep this result for comparability with prior causal-discovery studies,
but do not treat it as strong orientation evidence because the full Sachs
benchmark is fundamentally interventional.

\begin{table}[t]
  \centering
  \caption{Observational-only Sachs sanity check
    (11 nodes, 853 samples, 20-edge consensus DAG).}
  \label{tab:sachs_appendix}
  \begin{tabular}{lcccc}
    \toprule
    Method & SHD $\downarrow$ & F1 $\uparrow$ & TPR $\uparrow$ & Runtime (s) $\downarrow$ \\
    \midrule
    NOTEARS-Exp  & 19  & 0.286  & 0.300 & 4.45 \\
    polyDAG-Geo  & \textbf{19}  & \textbf{0.286} & \textbf{0.300} & \textbf{3.40} \\
    \bottomrule
  \end{tabular}
\end{table}

Both methods recover about 6/20 consensus edges from observational data only.
The modest F1 is expected given non-identifiability of many orientations
without interventions~\cite{peters2014identifiability,sachs2005causal}.

\section{Proof Details}
\label{app:proofs}

\subsection{Proof of Theorem~\ref{thm:poly}}

We provide additional detail for the forward direction ($\Rightarrow$).

Let $A = W \circ W \geq 0$ element-wise, with $A_{ii} = 0$ (no self-loops).
Suppose $\hpoly(W) = 0$, i.e., $\sum_{k=1}^d \tr(A^k) = 0$.
Since each $\tr(A^k) \geq 0$ (as $(A^k)_{ii} \geq 0$ when $A \geq 0$), we
have $\tr(A^k) = 0$ for all $k = 1, \ldots, d$.

For a nonneg\-ative matrix, $\tr(A^k) = 0$ implies $(A^k)_{ii} = 0$ for all $i$.
Consider any directed cycle $i_1 \to i_2 \to \cdots \to i_\ell \to i_1$ of
length $\ell \leq d$ in the graph encoded by $A$.
The product $A_{i_1 i_2} A_{i_2 i_3} \cdots A_{i_\ell i_1}$ would appear as a
term in $(A^\ell)_{i_1 i_1}$, which must therefore be positive.
But $\tr(A^\ell) = 0$ forces all diagonal entries of $A^\ell$ to be zero, a
contradiction.
Hence no cycle of length $\leq d$ exists, and the graph is acyclic.

\subsection{Derivation of the Geometric-Series Identity (Eq.~\eqref{eq:geo2})}

Let $S = \sum_{k=1}^d A^k$.  Then:
\begin{align*}
  (I - A) S &= (I - A)(A + A^2 + \cdots + A^d) \\
            &= A + A^2 + \cdots + A^d - A^2 - A^3 - \cdots - A^{d+1} \\
            &= A - A^{d+1}.
\end{align*}
When $(I - A)$ is invertible, $S = (I - A)^{-1}(A - A^{d+1})$, giving
Eq.~\eqref{eq:geo2}.

\subsection{Implementation Note on Gradients for $\hgeo$}

In our implementation, $\hgeo$ is evaluated as
\texttt{trace(solve(I - A, A - A\textasciicircum\{K+1\}))}, and gradients are obtained by
reverse-mode automatic differentiation through
\texttt{torch.linalg.solve} and \texttt{torch.matrix\_power}.
No separate hand-coded closed-form gradient is used in optimization.

For robustness under near-singular $(I-A)$, one can use stabilized variants
such as adaptive diagonal damping or matrix rescaling before the linear solve.
These variants were not part of the reported main experiments and are left for
future work.

\end{document}